\documentclass[twocolumn]{article}

\usepackage[a4paper,left=0.7in, right=0.7in,top=1.0in, bottom=1.0in]{geometry}
\usepackage{titlesec}

\titleformat{\section}
  {\normalfont\large\bfseries}
  {\thesection}
  {1em}
  {}

\titleformat{\subsection}
  {\normalfont\normalsize\bfseries}
  {\thesubsection}
  {1em}
  {}

\titlespacing*{\section}{0pt}{12pt plus 4pt minus 2pt}{6pt plus 2pt minus 2pt}

\usepackage{natbib}
\usepackage{microtype}
\usepackage{graphicx}
\usepackage{subfigure}
\usepackage{booktabs} 
\usepackage[dvipsnames]{xcolor}
\usepackage{subcaption}

\usepackage{hyperref}
\usepackage{bm}
\usepackage{enumitem}

\usepackage{authblk} 

\usepackage{amsmath}
\usepackage{amssymb}
\usepackage{mathtools}
\usepackage{amsthm}
\usepackage{thmtools}
\usepackage{thm-restate}

\usepackage[capitalize,noabbrev]{cleveref}
\Crefname{section}{Section}{Sections}
\Crefname{table}{Table}{Tables}
\Crefname{figure}{Figure}{Figures}

\theoremstyle{plain}

\theoremstyle{definition}

\theoremstyle{remark}

\usepackage[textsize=tiny]{todonotes}

\title{Learning Moderately Input-Sensitive Functions:\\ A Case Study in QR Code Decoding}

\author[1]{Kazuki Yoda}
\author[1]{Kazuhiko Kawamoto}
\author[1,2]{Hiroshi Kera\thanks{Corresponding author. Email: \href{mailto:kera@chiba-u.jp}{kera@chiba-u.jp}}}

\affil[1]{Chiba University}
\affil[2]{Zuse Institute Berlin}

\date{}
\begin{document}
\maketitle

\begin{abstract}
The hardness of learning a function that attains a target task relates to its input-sensitivity. For example, image classification tasks are input-insensitive as minor corruptions should not affect the classification results, whereas arithmetic and symbolic computation, which have been recently attracting interest, are highly input-sensitive as each input variable connects to the computation results. 
This study presents the first learning-based Quick Response (QR) code decoding and investigates learning functions of medium sensitivity.
Our experiments reveal that Transformers can successfully decode QR codes, even beyond the theoretical error-correction limit, by learning the structure of embedded texts. They generalize from English-rich training data to other languages and even random strings. Moreover, we observe that the Transformer-based QR decoder focuses on data bits while ignoring error-correction bits, suggesting a decoding mechanism distinct from standard QR code readers.
\end{abstract}

\section{Introduction}\label{submission}
Over a decade, deep learning has shown remarkable success in learning \textit{input-insensitive} functions that realize the target tasks. Indeed, in most standard tasks, such as image classification, object detection, document summarization, and speech recognition, a slight change in their input (e.g., image perturbation, a few typos) is supposed to have little impact on the output (e.g., classification, document summary). Consequently, for example, data augmentation techniques such as rotation, flipping, and cropping of images enhance the insensitivity (more commonly \textit{robustness}), which leads to substantial performance increases in these tasks.

Recent studies in arithmetic and symbolic computation further explore the learning of high-sensitivity functions with Transformer models~\cite{Transformer}. Examples include integer and modular arithmetic~\cite{grokking,Order}, symbolic integration~\cite{Symbolic}, Lyapunov function design~\cite{lyapunov}, Gr\"obner basis computation~\cite{Grobner,border_basis}, and so on~\cite{LWE,LWE3,LWE2,gcd}.
In such tasks, changing a single number, coefficient, variable, or operator in the input can immediately change the output, and thus, the target functions to learn are highly sensitive.
Although the theoretical difficulty of learning such functions has suggested their hardness~\cite{failure, Sensitive}, recent empirical and theoretical studies have shown that high-sensitivity functions can be learned through techniques such as normalization, weight decay, and autoregressive generation~\cite{normalize,weight_decay,Parity}.
As such, the sensitivity of the target function plays a critical role in determining the difficulty of learning, offering new insights into the capabilities of deep learning models.

\begin{figure}[t]
    \centering
    \includegraphics[width=\linewidth]{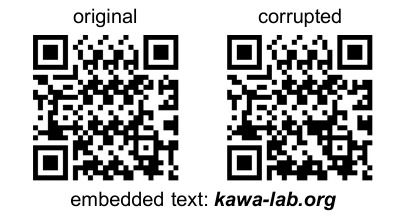}
    \vspace{-15pt}
    \caption{Example of a severely corrupted QR code successfully decoded by the Transformer. ``original'' denotes the QR code before corruption, and ``corrupted'' refers to the QR code after applying a 20-bit error. Standard QR code readers fail to decode the corrupted QR code, but the Transformer succeeds.
    The readers are recommended to try a QR code scan using their smartphone.}
    \label{fig:sample_fig1}
\end{figure}

In this study, we examine the intermediate case---learning medium-sensitivity functions---through the decoding task of Quick Response codes (QR codes;~\citet{ISO}) using a Transformer.
QR codes are designed to be tolerant (i.e., insensitive) to corruptions from image capture or physical damage, but decoding them should sensitively capture the change in the plain texts (i.e., embedded strings, such as URLs).
Such characteristics place QR code decoding between the low-sensitivity case of classical tasks and the high-sensitivity case of arithmetic and symbolic tasks. 
QR codes have several error-correction levels and embedding capacities through which we can control their sensitivity. 
To the best of our knowledge, no prior research has addressed learning-based QR code decoding. 
There have been many studies for better QR code detection and scanning as a computer vision task~\cite{CNN-QR,Angle-QR,Deblur,Faster_R-CNN,SRGAN,ESR-GAN,Reconstruction,LightWeight}, but these are orthogonal to our context.

Our experiments demonstrate a striking success of Transformer-based decoding under corruptions---the success rate of decoding under corruptions exceeds the theoretical limit, outperforming the standard decoding protocol. \cref{fig:sample_fig1} shows a case where a standard QR code reader fails to decode the QR code due to severe corruption, whereas the Transformer successfully decodes it. The datasets for training and evaluation consist of QR codes that embed domain names combining random English words and the top 1,000,000 domain names from the Tranco~\cite{Tranco}. Transformers can retrieve domain names from QR codes even under severe corruptions because of the intrinsic structure of natural language words (e.g., consonants and vowels roughly appear alternatingly in a word). Surprisingly, the trained Transformer generalizes to non-English words and even random alphabetic strings. Interestingly, the least successful case is the generalization to domain names with irregular words (e.g., ``freedoz'' instead of ``freedom''). In such a case, the Transformer model outputs spell-checked words (i.e., ``freedom''), which are incorrect.
We also investigated the sensitivity of the functions learned by the Transformer.
The Transformer successfully learned functions with four different sensitivity levels, each determined by the error-correction level, demonstrating its ability to decode QR codes.
In experiments where the location of errors was controlled, we found that the Transformer learned a function insensitive to the error-correction bits.
These findings suggest that the Transformer performs error correction using a mechanism distinct from standard QR code readers.

To summarize, this study addresses a novel QR code decoding task as a showcase of learning medium-sensitivity functions. We empirically obtain the following results through extensive experiments on domain name datasets:
\begin{itemize}[itemsep=5pt, parsep=0pt, topsep=0pt]
    \item Transformer-based decoding empirically attains a high success rate with low corruption and maintains moderate success even when the corruption magnitude exceeds the theoretical limit, which we newly derived for analysis, of error correction.
    \item Transformers learn natural language structure from English-rich datasets and generalize not only to other languages but even to random alphabetic strings.
    \item Transformers learn a function sensitive to information essential for decoding while remaining insensitive to redundant information.
    This tendency suggests that the Transformer performs error correction through a mechanism different from standard QR code readers, which rely explicitly on redundant information.
\end{itemize}

\section{Related Work}
\paragraph{Learning High-Sensitivity Functions.}
Transformers are widely used and excel at learning low-sensitivity functions like image classification~\cite{ViT}, object detection~\cite{Co-DETR}, and speech recognition~\cite{Speech}. 
Recent work shows Transformers can handle high-sensitivity tasks such as arithmetic and symbolic computation~\cite{Symbolic,Algebra,LWE,LWE3,LWE2,gcd,lyapunov,Grobner,border_basis}.
For instance, \cite{Symbolic} demonstrated that Transformers surpass established computational software such as Mathematica and Matlab in accuracy and computational efficiency when solving integrals and ordinary differential equations. 
Although training on high-sensitivity functions was long considered theoretically difficult~\cite{failure, Sensitive}, recent empirical work has shown that techniques such as normalization~\cite{normalize}, weight decay~\cite{weight_decay}, and autoregressive generation~\cite{Parity} can successfully surmount these challenges.
In previous research, functions with low or high sensitivity have been the focus, leaving medium-sensitivity functions almost unexplored.
In this study, we examine Transformer behavior by training models to decode QR codes, thereby probing a function of medium sensitivity.

\paragraph{Deep Learning Approaches to QR Codes.}
Many studies have explored applying deep learning to QR codes~\cite{CNN-QR,Angle-QR,Deblur,Faster_R-CNN,SRGAN,ESR-GAN,Reconstruction,LightWeight}.
For instance, several focus on enhancing detectability by leveraging deep learning models to accurately locate QR codes within images~\cite{CNN-QR,Angle-QR,Faster_R-CNN,LightWeight}.
On the other hand, other approaches use deep learning models to restore image quality and reduce blurring or low resolution to improve recognition accuracy~\cite{Deblur,SRGAN,ESR-GAN,Reconstruction}.
QR code reading consists of two principal stages: detection and decoding.
Prior works have enhanced reading performance by proposing methods to facilitate detection. 
In contrast, this study aims to analyze the characteristics of the Transformer by training it for decoding.

\begingroup
\setlength{\tabcolsep}{10pt}
\renewcommand{\arraystretch}{1.3}
\begin{table*}[t]
\centering
\caption{Success rate (\%) of a model trained on data in which the mask pattern was automatically selected based on the scoring rule, simulating a realistic scenario. The table also shows the proportion (\%) of mask patterns in the training dataset. When the scoring rule selects the mask pattern, it tends to favor patterns 1 and 4, resulting in a lower success rate for the other mask patterns.}
\label{tab:mask_mix}
\vskip 0.05in
\begin{tabular}{cccccccccc}
\hline
Mask Pattern &  0 &  1 &  2 &  3 &  4 &  5 &  6 &  7 &  Average \\ \hline
Success Rate &  59.6 &  92.8 &  49.1 &  64.2 &  90.9 &  50.7 &  68.8 &  54.2 &  68.3 \\
Proportion & 1.3 & 60.6 & 1.7 & 2.7 & 29.1 & 1.2 & 2.4 & 1.1 & - \\
\hline
\end{tabular}
\end{table*}
\endgroup

\section{QR Codes} \label{background}
QR codes~\cite{ISO} are two-dimensional matrix codes developed by Denso Wave in 1994 and designed for high-speed information reading.
QR codes enable high-speed, accurate transmission of textual information and are used worldwide for various applications, including website access, electronic payments, and airline ticketing.

\begin{figure}[t]
\centering
\includegraphics[width=\linewidth]{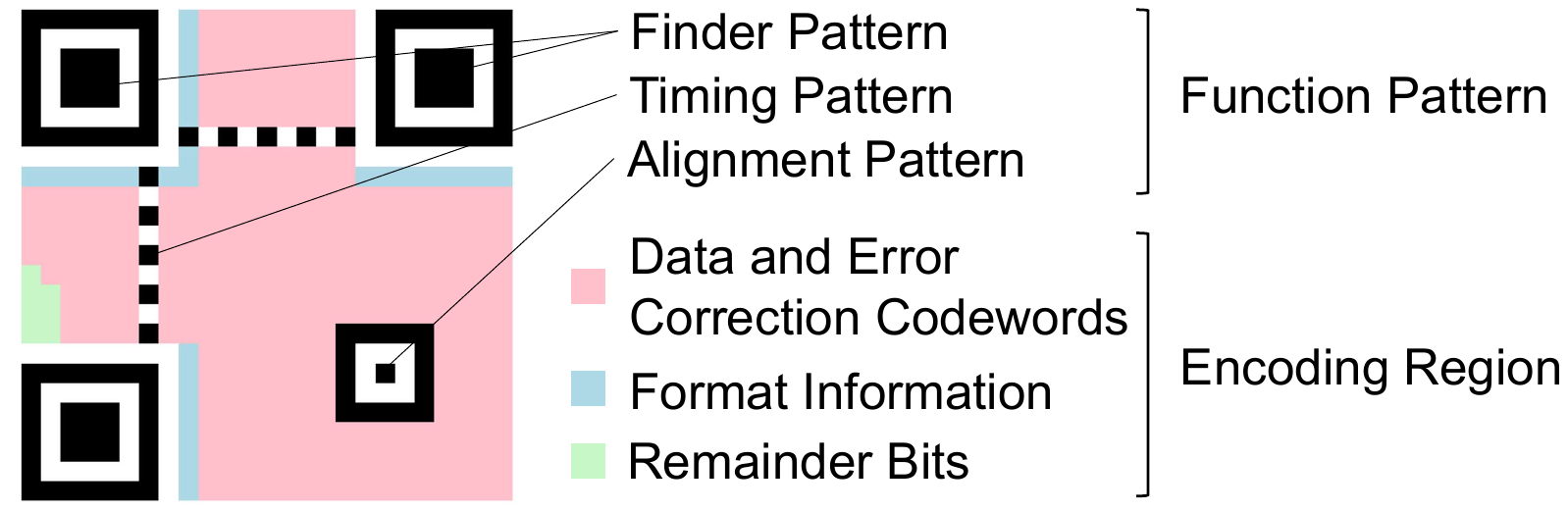}
\vspace{-15pt}
\caption{Structure of a v2-QR code. It comprises functional patterns, which support detection, and the encoding region, which embeds the encoded representation of the plain text.}
\label{fig:structure}
\end{figure}

\cref{fig:structure} illustrates the basic structure of a QR code. A QR code consists of black and white units (called \emph{modules}). 
Some of the units form \emph{function patterns}, which are used at the detection stage by cameras. 
The rest of the part corresponds to \emph{encoding region}, which encodes the data, along with error correction bits. The encoding region is divided into three components.
\paragraph{Data and Error-Correction Codewords.}
The plain text is encoded with the Reed–Solomon code and organized into 8-bit units called \emph{codewords}.
The data codewords represent the plain text itself, while the error-correction codewords contain the redundant information required for error correction. 
QR codes support four error correction levels (i.e., L, M, Q, and H in ascending order).
As the level increases, the proportion of error-correction codewords to the total number of codewords also increases.
\paragraph{Format Information.}
It encodes the error-correction level and mask pattern identifier using a BCH code. 
The resulting bit sequence is duplicated, with one copy placed next to the top-left finder pattern. 
The other copy is split between the regions next to the bottom-left and top-right finder patterns.
\paragraph{Remainder Bits.}
These bits are the leftovers resulting from the codeword placement. 
They contain no actual data or error-correction information and are ignored during decoding.

A QR code can have an imbalanced distribution of black and white modules. To mitigate this, a mask pattern is superposed onto the QR code in the final encoding step. There are eight mask patterns (\cref{fig:mask_pattern}), and the most suitable one is selected according to a scoring rule that penalizes poor arrangements, such as long runs of the same color or layouts that hinder detection. 

QR codes are defined in 40 versions, ranging from Version 1 to Version 40 by its size. The higher version has more modules, implying higher data capacity. For example, Version 1 has 21 modules per side, Version 2 has 25, and Version 3 has 29.
These versions are widely used in product packaging and promotional materials.

In what follows, we write QR codes, for example, of Version 3 and error correction level L as (v3, L)-QR codes.

\begin{figure}[t]
\centering
\includegraphics[width=\linewidth]{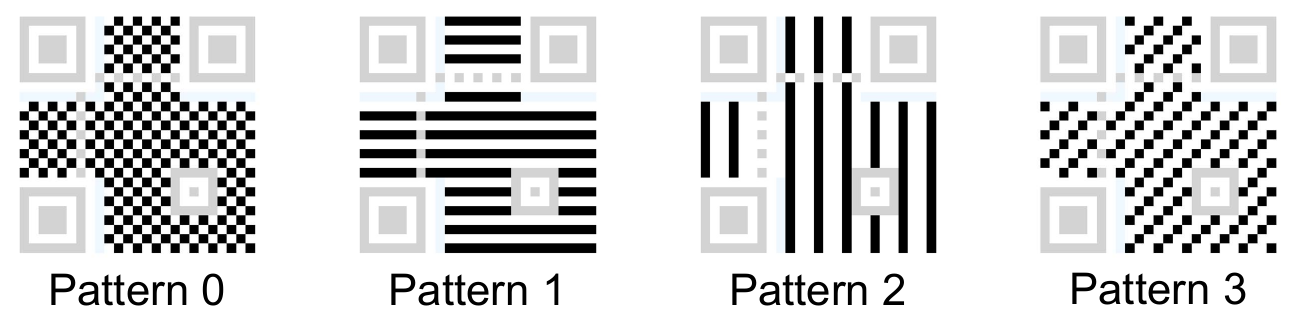}
\vspace{-15pt} 
\caption{Mask patterns 0 to 3. There are eight mask patterns in total. During masking, black modules within the designated pattern area are inverted. See \cref{appendix:mask_pattern} for all mask patterns.}
\label{fig:mask_pattern}
\end{figure}

\section{Success Rate Analysis of Error Correction}
We will later evaluate the robustness of QR code decoding by Transformers. 
To this end, we here derive the success rate of error correction in QR codes with $n$-bit errors. We assume that a bit flip occurs uniformly at random in the encoding region. To the best of our knowledge, this study is the first to derive the theoretical success rate.

Recall that the encoding region consists of the data and error-correction codewords, the format information, and the remainder bits, and the remainder bits are not used in the decoding. Let $N$ denote the total number of bits in the encoding region, and let $N_{\mathrm{d}}$, $N_{\mathrm{f}}$ and $N_{\mathrm{r}}$ be the numbers of bits in the data and error-correction codewords and the format information, remainder bits, respectively. Accordingly, the numbers of erroneous bits are denoted by $p$, $q$, and $n-p-q$, respectively. Then, the successful rate of error correction in encoding region is given by the following.
\begin{restatable}[name=Success Rate of Error Correction in the Encoding Region]{theorem}{ECCProbability}\label{theorem:ecc_probability}
Let  $n$ be the total number of bit errors in the encoding region.  
Then, the success rate of error correction —denoted by  $P_{\mathrm{success}}(n)$ —is given by:
\begin{align}\label{prob_total}
P_{\mathrm{success}}(n)
&= \frac{1}{\binom{N}{n}}
  \sum_{p=0}^{n}\sum_{q=0}^{n-p}
    W(p,q)P_{\mathrm{d}}(p)P_{\mathrm{f}}(q),
\end{align}
where $P_{\mathrm{d}}$ and $P_{\mathrm{f}}$ respectively denote the success rate of error correction in the data and error-correction codewords with $p$-bit errors and that in format information with $q$-bit errors, and 
\begin{align}
    W(p, q)=\binom{N_{\mathrm{d}}}{p}\binom{N_{\mathrm{f}}}{q}\binom{N_{\mathrm{r}}}{n-p-q}
\end{align}
is the probability that $(n-p-q)$-bit errors fall into the remainder bits.
\end{restatable}
In the following, we provide an overview of the derivation of $P_{\mathrm{d}}(p)$ and $P_{\mathrm{f}}(q)$. All the proofs can be found in \cref{appendix:success_rate}.

\paragraph{$\bm{P_{\mathrm{d}}(p)}$ - Success Rate in Data and Error-Correction Codewords.}
For the data and error-correction codewords, let $M = N_\mathrm{d} / 8$ denote the total number of codewords, and let $M_{\mathrm{ecc}}$ represent the number of error-correction codewords, which depends on the error correction level. 
According to the properties of Reed–Solomon codes, the maximum number of correctable codewords, denoted by $t$, is determined by $t = \lfloor M_\mathrm{ecc} / 2 \rfloor$.
Under these conditions, the success rate of error correction for the data and error-correction codewords can be determined as follows: 
\begin{restatable}[name=Success Rate of Error Correction in Data and Error-Correction Codewords]{theorem}{ProbabilityData}\label{theorem:prob_data}
Let $p$ be the number of erroneous bits within the data and error-correction codewords. Then, the success rate of error correction —denoted by $P_{\mathrm{d}}(p)$— is given by
\begin{align}
\begin{split}
    P_{\mathrm{d}}(p)
    &= \frac{1}{\binom{N_{\mathrm{d}}}{p}}\sum_{k = \lceil \frac{p}{8} \rceil}^{t} \binom{M}{k} \sum_{j=0}^{k}(-1)^{j} \binom{k}{j} \binom{8(k-j)}{p}.
\end{split}
\end{align}
\end{restatable}

\paragraph{$\bm{P_{\mathrm{f}}(q)}$ - Success Rate in Format Information.}
The format information is encoded into 15 bits using a $(15, 5)$ BCH code, and the resulting bit sequence is duplicated so that two identical copies are placed in the QR code.
If error correction works in either one of the two, the format information can be correctly read.
The $(15, 5)$ BCH code can correct up to 3-bit errors.
Under these conditions, the success rate of error correction for format information is given below:
\begin{restatable}[name=Success Rate of Error Correction in Format Information]{theorem}{ProbabilityFormat}
\label{theorem:prob_format}
Let $q$ denote the number of erroneous bits affecting the format information, and let $i$ and $j$ represent the number of bit errors in each of the two respective instances.
Then, the success rate of error correction —denoted by $P_\mathrm{f}$—is given by
\begin{align}
P_{\mathrm{f}}(q)
&= \frac{\bigl|\{(i,j)\in\mathbb{N}_0^2 \mid i+j=q,\ \min(i,j)\le3\}\bigr|}{q+1}.
\end{align}
\end{restatable}

\section{Learning to Decode QR Codes} \label{learning}
In this section, we present the evaluation results of the Transformer's QR code decoding success rate and robustness. 
We assume that QR code detection has been successfully performed and focus solely on the decoding phase.

\subsection{Setup} \label{setup_decode}
\paragraph{Task.}
Transformers are trained for the QR code decoding task.
The input is a QR code, and the output is the plain text embedded in the QR code.
The QR code is passed to the model in the bit string format.

\paragraph{Dataset.}
We sampled domain names from Tranco~\cite{Tranco}, a publicly available ranking of popular domain names.
The sampled domain names were encoded into QR codes using the Segno library\footnote{\url{https://github.com/heuer/segno}} as a QR code generator.
Each QR code was then linearized into a one-dimensional sequence following \cref{fig:flatten}(d).
We fixed each dataset's QR code version, error correction level, and mask pattern.
We generated 500,000 samples for the training dataset and 1,000 samples for the evaluation dataset.

\paragraph{Model \& Training.}
To evaluate the general capabilities of the Transformer, we adopted a standard architecture~\cite{Transformer} (six encoder and decoder layers, eight attention heads) and conventional training settings (AdamW optimizer~\cite{AdamW} with a linearly decaying learning rate starting from $10^{-4}$).
Following~\cite{BART}, we applied weight sharing between the input embeddings and the output projection layer and employed learnable positional embeddings.
The batch size was set to 16, and training was conducted for 10 epochs.

\paragraph{Evaluation.}
The decoding success rate is defined as the proportion of successfully decoded samples in the evaluation dataset.
A decoding is considered successful if the Transformer's output exactly matches the plain text.
To evaluate robustness, we compute the decoding success rate on evaluation datasets composed of artificially corrupted QR codes.
We introduce two types of artificial corruption: flip errors and burst errors.
Flip errors are introduced by randomly selecting bits within the QR code and inverting their binary values.
Burst errors are generated by randomly selecting a 3×3 square region and forcing all bits within that region to be 1.
In this experiment, we assume that the detection stage has already been completed.
Therefore, corruption is applied only to the encoding region, excluding functional patterns.
For evaluation on corrupted data, we compare the Transformer with pyzbar\footnote{\url{https://github.com/NaturalHistoryMuseum/pyzbar}}—a Python wrapper for the widely used open-source ZBar QR code decoder\footnote{\url{https://github.com/mchehab/zbar}}.

\begin{figure}[t]
    \centering
    \includegraphics[width=\linewidth]{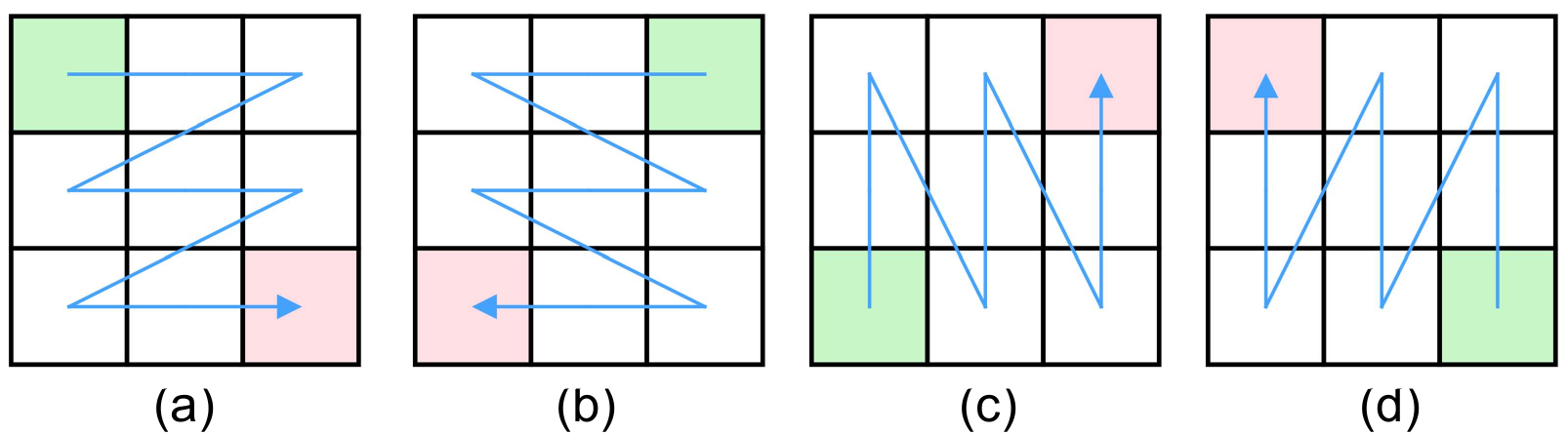}
    \vspace{-15pt}
    \caption{Candidate linearization orders for converting a QR code into a one-dimensional sequence}
    \label{fig:flatten}
\end{figure}

\begingroup
    \setlength{\tabcolsep}{10pt}
    \renewcommand{\arraystretch}{1.3}
    \begin{table}[t]
        \centering
        \caption{Success rate (\%) for four different linearizing orders in (v3, L)-QR codes. The ordering shown in \cref{fig:flatten}(d) outperforms the others, suggesting it is the most suitable for training Transformers.}
        \label{tab:flatten}
        \vskip 0.05in
        \begin{tabular}{ccccc}
        \hline
        Order & (a) & (b) & (c) & (d) \\ 
        \hline
        Success Rate & 93.3 & 90.2 & 93.9 & $\bm{95.5}$ \\
        \hline
        \end{tabular}
    \end{table}
\endgroup

\paragraph{Preliminary Experiments.}
The aforementioned experimental setup is based on three observations from our preliminary experiments. 
\begin{enumerate}[itemsep=2pt, parsep=0pt, topsep=1pt, partopsep=0pt ]
    \item It was unsuccessful to train Transformers on data where the mask pattern was automatically selected, which reflects a realistic scenario (\cref{tab:mask_mix}). 
    This is because the proportion of mask patterns becomes highly imbalanced. The details can be found in \cref{appendix:mask_mix}.
    \item It was easy to train a near-perfect classifier that classifies the mask pattern from input QR codes with almost 100\,\% accuracy, see \cref{appendix:mask_classification}. 
    \item Transformer trained on bit strings in the ordering in \cref{fig:flatten}(d) was more successful than in others (\cref{tab:flatten}).
    This is likely because the ordering in \cref{fig:flatten}(d) more closely matches the bit ordering of QR codes, as illustrated in \cref{fig:order}.
\end{enumerate}

From the first two observations, we decided to focus on a fixed mask pattern. From the last observation, we adopted the ordering in \cref{fig:flatten}(d) in our main experiments.

\begin{figure}[t]
    \centering
    \includegraphics[width=\linewidth]{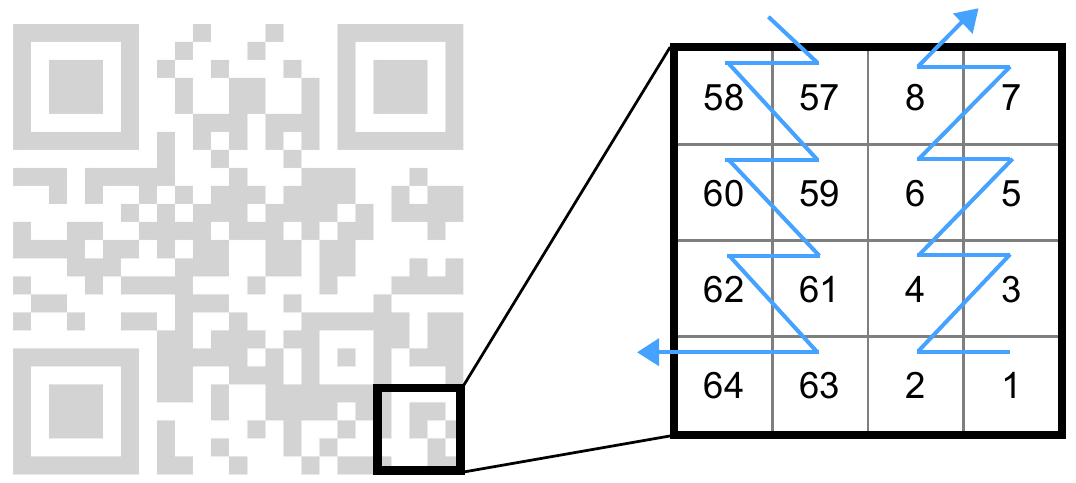}
    \vspace{-15pt}
    \caption{Arrangement order of encoded data in v2-QR Code. The data is placed in a vertical zigzag pattern, moving up and down in two-column bands while avoiding function patterns.}
    \label{fig:order}
\end{figure}

\begingroup
\setlength{\tabcolsep}{10pt}
\renewcommand{\arraystretch}{1.3}
\begin{table*}[t]
    \centering
    \caption{Success rate (\%) of Transformers for each mask pattern in (v1, L) to (v3, L)-QR codes. Across all versions, the model achieves an average decoding success rate of over 93\,\%, suggesting that Transformers can learn functions with medium sensitivity. Here, as shown in \cref{background}, the version determines the number of bits and the structural layout of the QR code; the error correction level represents the degree of redundancy for data recovery; and the mask pattern specifies the type of mask applied after encoding.}
    \label{tab:accuracy}
    \vskip 0.05in
    \begin{tabular}{cccccccccc}
    \hline 
    Mask Pattern &  0 &  1 &  2 &  3 &  4 &  5 &  6 &  7 &  Average \\ \hline 
    Version 1 &  99.0 &  97.9 &  99.2 &  97.8 &  98.6 &  98.5 &  97.5 &  97.8 &  98.3 \\
    Version 2 &  96.0 &  95.5 &  96.3 &  95.0 &  92.1 &  95.0 &  94.5 &  94.3 &  94.8 \\
    Version 3 &  95.5 &  95.5 &  95.6 &  93.6 &  95.1 &  91.6 &  93.4 &  90.8 &  93.9 \\
    \hline 
    \end{tabular}
\end{table*}
\endgroup

\begin{figure*}[t]
    \centering
    \includegraphics[width=\linewidth]{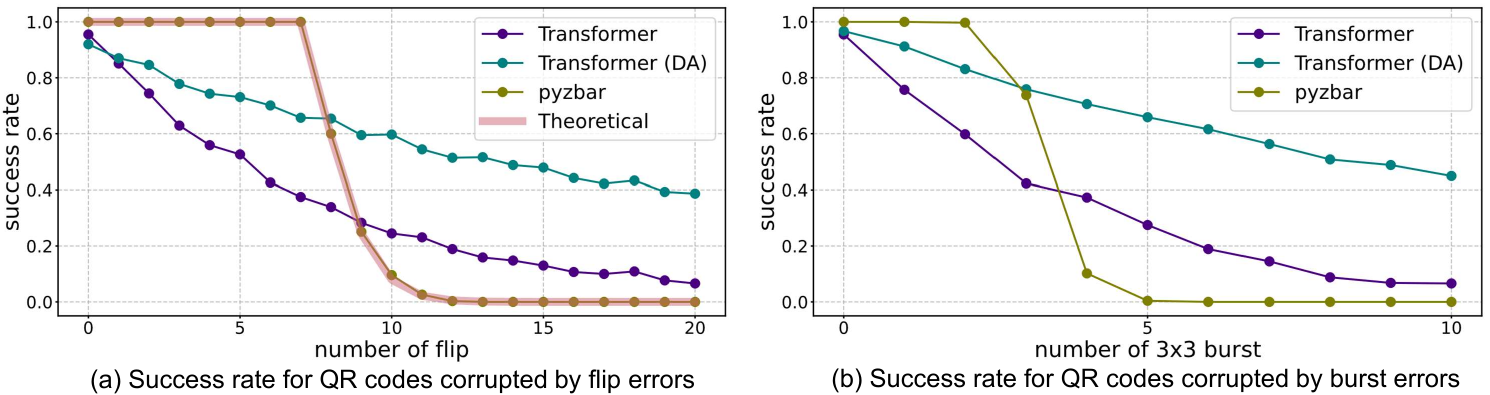}
    \vspace{-15pt}
    \caption{Success rate on corrupted (v3, L)-QR codes (mask pattern 0) for each method. (a) Success rate under flip noise. (b) Success rate under burst noise. ``Transformers (DA)'' refers to the model trained with corrupted data augmentation. ``Theoretical'' refers to the success rate of error correction proved in \cref{theorem:ecc_probability}. When the level of corruption is high, the Transformer can decode with a greater success rate than the success rate guaranteed by the QR code's error correction capability. In addition, applying data augmentation leads to substantial improvements in error resilience.}
    \label{fig:graph}
\end{figure*}

\subsection{Decoding Success Rate and Robustness}
We evaluated the Transformer's ability to decode QR codes by decoding success rate, robustness to corruption, and the effect of data augmentation. 
We also looked at the model's outputs when decoding both clean and corrupted QR codes.

\paragraph{Decoding Success Rate.}
\cref{tab:accuracy} shows the decoding success rate of (v1, L) to (v3, L)-QR codes under each mask pattern.
Across all versions, the average success rate exceeded 93\,\%, indicating high performance.
These results suggest that Transformers can learn medium-sensitivity functions.

\paragraph{Decoding Robustness.}
\cref{fig:graph} illustrates the evaluation results for (v3, L)-QR codes (mask pattern 0) with two types of corruption.
\cref{fig:graph}(a) shows that the Transformer achieves a higher success rate than the theoretical success rate of error correction proven in \cref{theorem:ecc_probability} when the number of flip errors exceeds nine.
Furthermore, the Transformer surpassed the pyzbar when the number of flip errors exceeded nine or burst errors exceeded four.
This result indicates that the Transformer exhibits superior robustness under severe corruption.
Additionally, in \cref{appendix:severe}, we present examples of severe corruption where pyzbar fails to decode the QR code but the Transformer succeeds (\cref{appendix:sample_flip,appendix:sample_burst}).
Such robustness may be attributed to the model's ability to learn regularities in word patterns commonly found in domain names, thereby enabling it to infer and reconstruct corrupted information.
Note that the plain texts used in the training and evaluation datasets differ.

\paragraph{Data Augmentation.}
We trained the Transformer on (v3, L)-QR codes (mask pattern 0) augmented with flip and burst errors.
As shown in \cref{fig:graph}, it improved significantly on corrupted data.
These results suggest that training on corrupted inputs can make models more robust.

\paragraph{Examples of Generated Strings.}
Examples of strings generated during decoding by the Transformer are shown in \cref{tab:example_generation}.
When the generated string matches the original string precisely, the model successfully handled both short and long plain text.
In addition, even when the generated strings did not precisely match the original string, they were generally close to the ground truth. 
For instance, for ``dtv2009.gov,'' the model outputs ``dtv2008.gov,'' which differs only by one character. 
There are also cases where the model makes mistakes because it has learned specific patterns in English words.
For example, for the plain text ``doggettinc.com,'' the Transformer incorrectly generated ``doggetting.com.''
This was likely because it misinterpreted ``inc'' as part of the preceding word and segmented it as ``getting,'' reflecting a misunderstanding influenced by learned linguistic regularities.
On the other hand, \cref{tab:example_generation}(b) shows examples of outputs generated from inputs with 20 random bit flips.
Compared to the clean examples in \cref{tab:example_generation}(a), the outputs are generally less similar to the original strings.
However, some errors affect only one character, such as ``delicom.global'' becoming ``dedicom.global.''
As shown in \cref{fig:graph}(a), pyzbar fails to return any output when 20-bit corruption is applied.
This is because it is designed to suppress output entirely once the corruption exceeds a certain threshold to avoid incorrect results.
In contrast, the Transformer always produces an output, even if it is incorrect.
While this can lead to errors, the outputs often remain close to the original, even with heavy corruption.
Additionally, in \cref{appendix:severe}, we present the distribution of similarity between the generated strings and the original plain text under severe corruption (\cref{appendix:fig:similarity}).

\begin{table*}[t]
\renewcommand{\arraystretch}{1.2}
\centering
\caption{Examples of strings generated by the Transformer from clean and corrupted (v3, L)-QR codes (mask pattern 0). For clean inputs (a), the model successfully reconstructs the plain text with a high success rate for both short and long sequences. Even when decoding fails, the generated strings exhibit high similarity to the ground truth. For inputs with 20-bit flips (b), lower-similarity failures are observed. However, despite severe corruption, both successful cases and failures are also present, with high similarity to the ground truth.}
\label{tab:example_generation}
\vskip 0.05in
\begin{minipage}[t]{0.48\textwidth}
\centering
\text{(a) Clean QR codes}
\vskip 0.05in
\begin{tabular}{ccc}
\hline
& Ground Truth & Prediction \\ \hline
$\checkmark$ & ellis.ru & ellis.ru \\
$\checkmark$ & mobile-arsenal.com.ua & mobile-arsenal.com.ua \\
$\checkmark$ & osakabasketball.jp & osakabasketball.jp \\ \hline
$\times$ & doggettin\textcolor{red}{c}.com & doggettin\textcolor{red}{g}.com \\
$\times$ & elalman\textcolor{red}{a}que.com & elalman\textcolor{red}{i}que.com \\
$\times$ & dtv200\textcolor{red}{9}.gov & dtv200\textcolor{red}{8}.gov \\ \hline
\end{tabular}
\end{minipage}
\hfill
\begin{minipage}[t]{0.48\textwidth}
\centering
\text{(b) QR codes with 20-bit flips}
\vskip 0.05in
\begin{tabular}{ccc}
\hline
& Ground Truth & Prediction \\ \hline
$\checkmark$ & casino-x-dawn.bet & casino-x-dawn.bet \\
$\checkmark$ & mobil-isc.de & mobil-isc.de \\
$\checkmark$ & slotsyps.info & slotsyps.info \\ \hline
$\times$ & de\textcolor{red}{l}icom.global & de\textcolor{red}{d}icom.global \\
$\times$ & inh\textcolor{red}{and}.com & inh\textcolor{red}{cdn}.com \\
$\times$ & mob\textcolor{red}{i}le-arsena\textcolor{red}{l}.com.\textcolor{red}{ua} & mob\textcolor{red}{x}le-arsena\textcolor{red}{t}.com.\textcolor{red}{mt} \\
\hline
\end{tabular}
\end{minipage}
\end{table*}

\section{Generalization performance of decoding with Transformer} \label{generalization}
In this section, we evaluate the generalization performance of the Transformer in QR code decoding on various types of plain text.

\subsection{Setup}
The training configuration is identical to that in \cref{setup_decode}.  
To evaluate the Transformer's generalization, we generated eight evaluation sets.
Each evaluation sample follows the domain name format
\emph{word}\textsubscript{1}\emph{word}\textsubscript{2}\texttt{.}\emph{tld},
where \emph{word}\textsubscript{1} and \emph{word}\textsubscript{2} are concatenated to form the second-level domain (SLD), and \emph{tld} denotes the top-level domain (TLD), which is randomly selected from frequent options such as ``com,'' ``org,'' or ``co.''
We constructed eight evaluation sets, each comprising 5,000 (v3, L)-QR codes (mask pattern 0) generated from domain names with varying lexical or structural properties. The eight evaluation datasets are as follows:

\begin{itemize}[itemsep=1pt, parsep=0pt, topsep=0pt]
\item \textbf{English, German, Swahili}: words in the respective language. 
\item \textbf{Shuffle}: words created by randomly permuting the characters of English words.
\item \textbf{Random-alphabet}: random sequences consisting of alphabet characters.
\item \textbf{Misspelled}: English words in which one character is randomly replaced with an incorrect alphabet.
\item \textbf{Leetspeak}: English words rewritten using visually similar numerals (e.g., c4t; originally cat).
\item \textbf{no-TLD}: two English words concatenated without a top-level domain (e.g., pianofox instead of pianofox.com). This structural deviation removes the period and TLD component, setting it apart from the standard domain format used in the other datasets.
\end{itemize}

\begingroup
    \setlength{\tabcolsep}{6pt}
    \renewcommand{\arraystretch}{1.3}
    \begin{table*}[t]
        \centering
        \caption{Success rate (\%) across eight datasets. In natural language datasets such as English, German, and Swahili, the Transformer shows a high success rate, indicating that the Transformer effectively leverages linguistic regularities. On the other hand, datasets lacking such structure (e.g., Shuffle, Random-alphabet) or containing minor perturbations (e.g., Misspelled, Leetspeak) result in a lower success rate.}
        \label{tab:generalization}
        \vskip 0.05in
        \begin{tabular}{ccccccccc}
        \hline
        Dataset  & English & German & Swahili & Shuffle & Random-alphabet & Misspelled & Leetspeak & no-TLD \\
        \hline
        Success Rate (\%)    & 99.5    & 97.1      & 96.6   & 95.3  & 94.4  & 88.1      & 72.5      &  3.0     \\
        \hline
        \end{tabular}
    \end{table*}
\endgroup

\subsection{Generalization Performance on Multiple Data Sets}
\cref{tab:generalization} presents the evaluation results on eight evaluation sets.
The Transformer's success rate varied depending on the characteristics of the dataset.

\paragraph{Natural Language Datasets.} 
In the English, German, and Swahili datasets, consisting of natural language, the Transformer achieved higher success rates than the others. 
Since the training data was created from popular domain names in the Tranco, many samples include English words. As a result, the Transformer likely learned common patterns such as frequently used prefixes and suffixes, as well as other frequent patterns like consonant–vowel alternation. 
The differences in success rates among English, German, and Swahili are likely due to differences in the structure of words across these languages.

\paragraph{Unstructured Datasets.}
Compared to the English dataset, the Transformer achieved 3.2\,\% lower success rates on the Shuffle dataset and 4.1\,\% lower on the Random-alphabet dataset.
This is likely because these datasets lack frequent patterns and regularities typically found in real words.

\paragraph{English-Variant Datasets.}
Although the Misspelled and Leetspeak datasets retain much of the English word structure, the Transformer achieved lower success rates than the English dataset.
Specifically, the success rate dropped by 11.4\,\% on the Misspelled dataset and by 27.0\,\% on the Leetspeak dataset.
These results suggest that even small changes can cause the Transformer to misread letters and instead predict other characters based on learned patterns.
As shown in \cref{appendix:misspelled}, we identified examples in the Misspelled dataset where the model's learned regularities of English words led to failures.
The extremely low success rate on the Leetspeak dataset may be attributed to the lower frequency of digits than letters in the training dataset.

\paragraph{no-TLD Dataset.} 
The success rate on the no-TLD dataset was notably low at 3.0\,\%. This can be attributed to the fact that the Transformer was trained exclusively on domain names, which caused it to append a top-level domain during inference in most cases.

The experiment shows that the Transformer learns language rules and frequent patterns during training. As a result, it performs well on datasets made from a natural language with a clear structure. 
In contrast, it showed a lower success rate on the Shuffle and Random-alphabet datasets, which lack such structure.
The Misspelled and Leetspeak datasets kept the basic structure of English words but included small changes. 
In these cases, the Transformer was often confused by the small changes and mistakenly treated them as clean patterns it had learned before.

\section{Sensitivity Analysis of Transformer's Behavior in QR Code Decoding}
In this section, we train Transformers on different error-correction levels (i.e., different input sensitivity) and examine their effect on the trainability of Transformers and the sensitivity after training.

\begin{figure}[t]
    \centering
    \includegraphics[width=\linewidth]{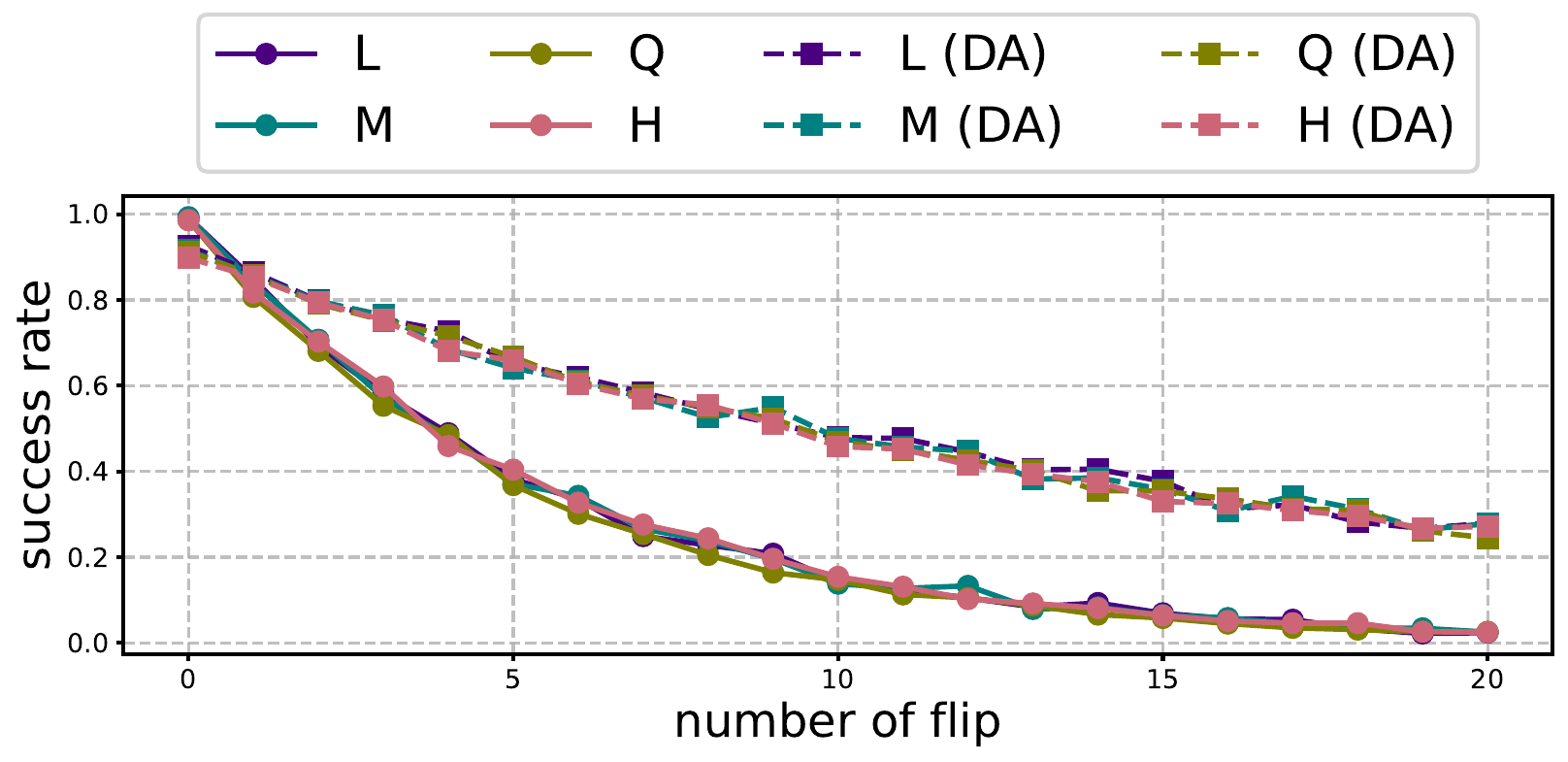}
    \vspace{-15pt}
    \caption{Success rates under (a) flip and (b) burst errors were evaluated for v2-QR codes with all four error correction levels. ``(DA)'' indicates the result when data augmentation is applied during training. There were almost no differences in success rates between the levels. }
    \label{fig:flip_ce_da}
\end{figure}

\subsection{Setup}
Recall that QR codes support four levels of error correction—L, M, Q, and H—in ascending order of strength, with each higher level able to correct more bit errors.
Moreover, as the level increases, the proportion of error-correction codewords to the total number of codewords also increases, which leads to lower sensitivity~(cf.~\cref{appendix:tab:sensitivity}).
We trained Transformers for the v2-QR code decoding task with each error-correction level.
All training and evaluation settings aside from the dataset were identical to those in \cref{setup_decode}.
The training and test sets were constructed from the same plain texts across different levels so that we could examine the impact of sensitivity differences.
Therefore, sensitivity differences appear only in the error-correction codewords.

\subsection{Results and Discussion on Sensitivity Behavior}
\cref{fig:flip_ce_da} shows the success rates of decoding v2-QR codes across the four error-correction levels.
The sensitivity difference did not impact the success rate. 
To understand this, we plot the success rate decay over the number of flip errors in \cref{fig:heterogeneous}. 
It reveals that the bit flips in data codewords decrease the success rate, while those in error-correction codewords do not. 
This tendency holds even when data augmentation is applied during training.
Namely, Transformers do not use error correction bits for QR code decoding.

\begin{figure}[t]
    \centering
    \includegraphics[width=\linewidth]{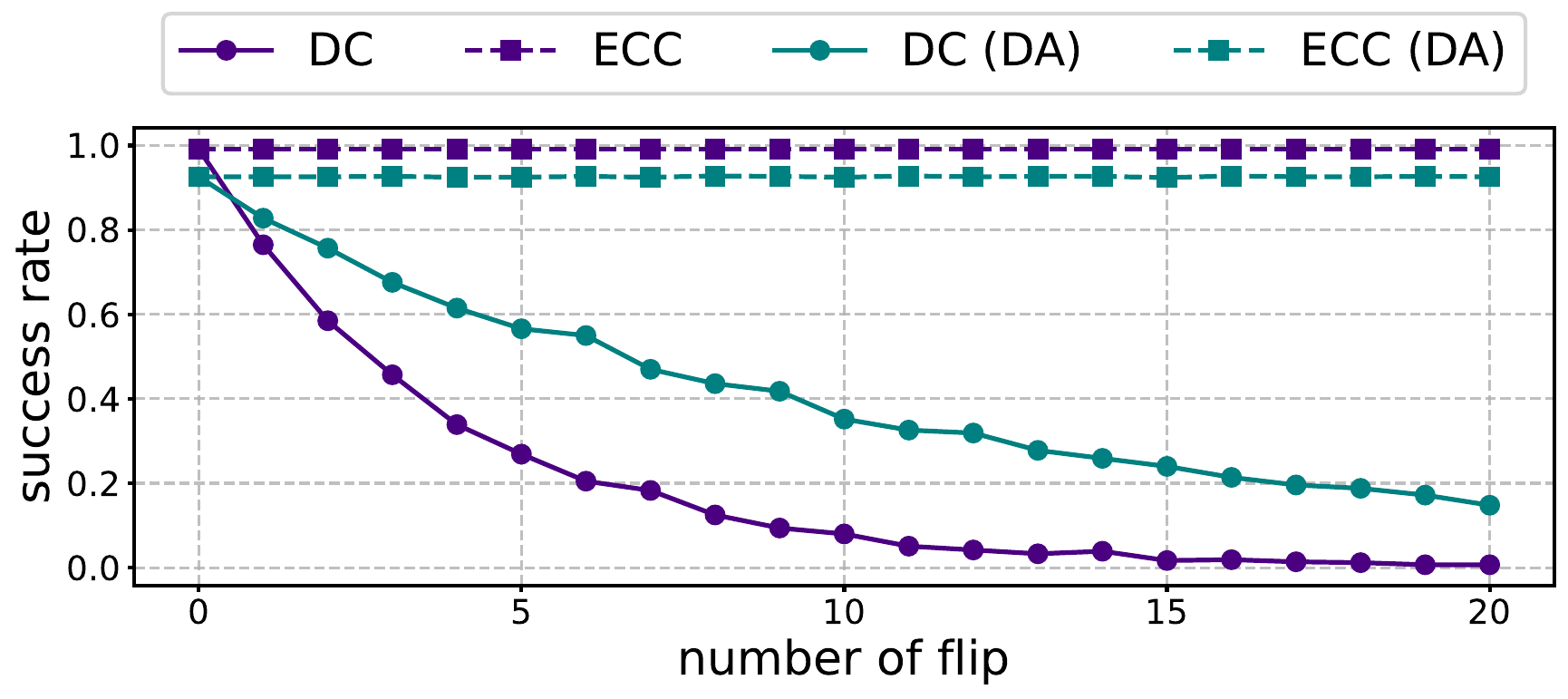}
    \vspace{-15pt}
    \caption{Success rates when bit-flip errors are applied either to the data codewords (DC) or the error-correction codewords (ECC) of (v2, L)-QR codes. ``(DA)'' indicates the result when data augmentation is applied during training. While errors in the error-correction codewords have little impact, those in the data codewords significantly degrade the success rate.}
    \label{fig:heterogeneous}
\end{figure}

\section{Conclusions}
This study tackles the QR code decoding task as an example of a medium-sensitivity function.
We demonstrate that Transformers achieve a high success rate under low corruption and maintain moderate decoding success rates even when the corruption exceeds the theoretical error-correction limit, which we newly derived.
Moreover, a Transformer trained on datasets rich in English words learns the structure of natural language and generalizes not only to other languages but also to random alphabetic strings.
In addition, we show that the Transformer learns a function that is insensitive to error-correction codewords.
This finding suggests that the Transformer performs error correction through a mechanism different from that of standard QR code readers.

While our findings provide insights into the Transformer's ability to decode QR codes as a medium-sensitivity function, this study does not investigate how its internal structure affects decoding performance. In future work, comparing architectures with different positional embeddings and attention mechanisms will be essential to gain a deeper understanding of the Transformer's behavior in decoding.

\section*{Acknowledgments}
This research was partially supported by JST PRESTO Grant Number JPMJPR24K4, Mitsubishi Electric Information Technology R\&D Center, the Chiba University IAAR Research Support Program and the Program for Forming Japan's Peak Research Universities (J-PEAKS), and JSPS KAKENHI Grant Number JP23K24914.

\bibliographystyle{plainnat}

\newpage
\appendix
\onecolumn

\section{Mask Pattern} \label{appendix:mask_pattern}
QR Code has eight different mask patterns~\cite{ISO}.
\cref{fig:mask_pattern_all} shows a list of mask patterns.

\begin{figure}[ht]
    \centering
    \includegraphics[width=\linewidth]{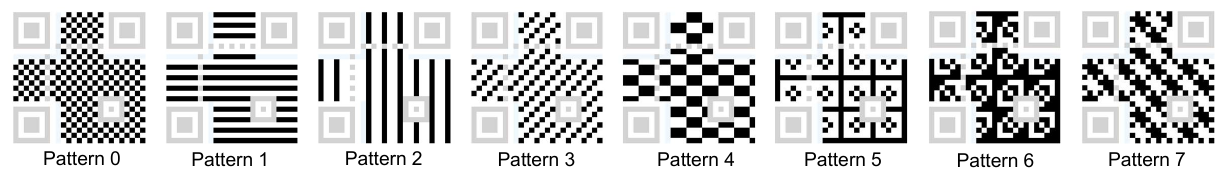}
    \vspace{-15pt}
    \caption{List of mask patterns.}
    \label{fig:mask_pattern_all}
\end{figure}

\section{Transformer Performance on Data with Mask Patterns Automatically Selected by the Scoring Rule} \label{appendix:mask_mix}
In QR codes, mask patterns are ordinarily selected automatically by the scoring criteria~\cite{ISO}. In this section, we construct the dataset by encoding v3-QR codes from internet-domain inputs with mask patterns chosen automatically by the scoring rules and subsequently train the Transformer. For evaluation, we prepare 1,000 samples for each of the eight mask patterns, yielding 8,000 samples. The experimental configuration follows that of Section 5.1. As presented in \cref{appendix:tab:mask_mix}, the Transformer attains an average decoding success rate of 68.3\,\% when mask patterns are heterogeneous. Although only patterns 1 and 4 achieve decoding accuracies above 90\,\%, the remaining patterns fall below the 3\,\% threshold. We attribute this disparity to the biased proportion of mask patterns within the training dataset. \cref{appendix:tab:mask_mix} depicts this proportion, revealing that most of the 500,000 training samples correspond to patterns 1 and 4. This pronounced imbalance results in an insufficient representation of the other patterns, impairing generalization and reducing inference performance.

\begingroup
\setlength{\tabcolsep}{10pt}
\renewcommand{\arraystretch}{1.3}
\begin{table*}[ht]
\centering
\caption{Success rate (\%) of a model trained on data in which the mask pattern was automatically selected based on the scoring rule, simulating a realistic scenario. The table also shows the proportion (\%) of mask patterns in the training dataset. When the scoring rule selects the mask pattern, it tends to favor patterns 1 and 4, resulting in a lower success rate for the other mask patterns. (Repeated From \cref{tab:mask_mix})} \label{appendix:tab:mask_mix}
\vskip 0.05in
\begin{tabular}{cccccccccc}
\hline
Mask Pattern &  0 &  1 &  2 &  3 &  4 &  5 &  6 &  7 &  Average \\ \hline
Success Rate &  59.6 &  92.8 &  49.1 &  64.2 &  90.9 &  50.7 &  68.8 &  54.2 &  68.3 \\
Proportion & 1.3 & 60.6 & 1.7 & 2.7 & 29.1 & 1.2 & 2.4 & 1.1 & - \\
\hline
\end{tabular}
\end{table*}
\endgroup

\section{Mask Pattern Classification via Transformer} \label{appendix:mask_classification}
To ascertain whether a Transformer can discriminate among mask patterns, we trained and evaluated the model using a dataset encompassing all eight mask patterns. For training, we curated 800,000 popular internet domains from the Tranco list and synthesized the dataset such that each mask pattern was represented by 100,000 samples. For evaluation, we collected 8,000 samples from internet domains distinct from the training set, ensuring each mask pattern comprised 1,000 samples. The architecture and training protocol were identical to those described in \cref{setup_decode}. We set the batch size to eight and trained the model for one epoch. The evaluation revealed that the Transformer achieved a classification accuracy of 100\,\%. This result confirms that mask patterns can be classified with perfect accuracy; consequently, in our proposed method, we assume known mask patterns and train separate models for each pattern.

\section{Generated Strings for Misspelled Dataset} \label{appendix:misspelled}
\cref{appendix:tab:Misspelled} shows examples of strings generated by the Transformer on the Misspelled dataset in the experiments of \cref{generalization}.
From the successful cases, it is clear that the Transformer can handle spelling mistakes. 
However, the failure examples reveal instances where its learned knowledge of English word patterns actually leads to incorrect outputs.
For example, the plain text ``domajinprotocjl.me'' is composed of misspellings of the two English words ``domain'' and ``protocol,'' but the Transformer generated ``domainprotocol.me.''
This appears to result from misreading ``domaij'' as ``domain'' and ``protocjl'' as ``protocol.''

\setlength{\tabcolsep}{15pt} 
\begin{table}[htbp]
\renewcommand{\arraystretch}{1.2}
  \centering
  \caption{Examples of strings generated by the Transformer on the Misspelled dataset in the experiments of \cref{generalization}.}
  \vskip 0.05in
  \label{appendix:tab:Misspelled}
  \begin{tabular}{ccc|c}
    \hline
     & Ground Truth & Prediction & Component Words\\ \hline
    $\checkmark$ & drwafcess.io      & drwafcess.io   & dry + access              \\
    $\checkmark$ & eglyiefd.jp       & eglyiefd.jp   & egg + yield                     \\
    $\checkmark$ & hopeynww.co       &  hopeynww.co  & honey + new    \\
    $\checkmark$ & joqpct.jp         &    joqpct.jp   & job + act        \\ 
    $\checkmark$ & yelloqnobe.io     &   yelloqnobe.io & yellow + note         \\ \hline
    $\times$     &  acrfreedo\textcolor{red}{z}.io  & acrfreedo\textcolor{red}{m}.io & act + freedom\\
    $\times$     & domai\textcolor{red}{j}protoc\textcolor{red}{j}l.me & domai\textcolor{red}{n}protoc\textcolor{red}{o}l.me & domain + protocol \\
    $\times$     & fiameupstre\textcolor{red}{r}m.site  & fiameupstre\textcolor{red}{a}m.site & flame + upstream\\
    $\times$     & gho\textcolor{red}{c}twnreless.fr &  gho\textcolor{red}{s}twnreless.fr & ghost + wireless\\ 
    $\times$     & ordbrow\textcolor{red}{q}er.cc  & ordbrow\textcolor{red}{s}er.cc & old + browser \\ \hline
  \end{tabular}
\end{table}
\setlength{\tabcolsep}{6pt}

\section{Transformer Performance under Severe Corruption} \label{appendix:severe}
In this section, we demonstrate the Transformer's success under severe corruption.
First, we present examples in which the Transformer successfully decodes heavily corrupted inputs. 
Recall that in this study, flip errors and burst errors were artificially simulated.
\cref{appendix:sample_flip,appendix:sample_burst} shows cases for both flip and burst errors where the Transformer succeeds in decoding while the pyzbar fails to read the code.
In each case, the corruption causes a significant deviation from the original QR code, preventing pyzbar from decoding it. However, the Transformer successfully decodes these corrupted codes.
The readers are recommended to try a QR code scan using their smartphone.

\begin{figure}[ht]
    \centering
    \includegraphics[width=\linewidth]{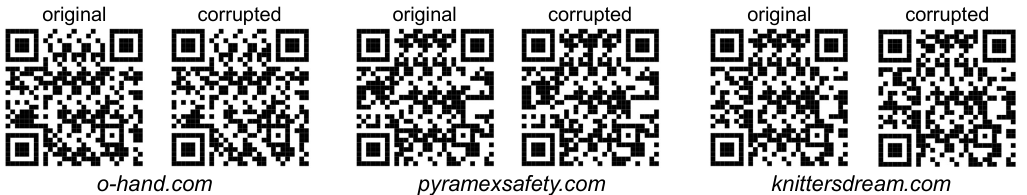}
    \vspace{-15pt}
    \caption{Examples where the Transformer successfully decodes QR codes after 20 flip errors.
``original'' shows the QR code before corruption, and ``corrupted'' shows the version after applying the errors.
While pyzbar fails to decode the corrupted QR code, the Transformer can decode it.}
    \label{appendix:sample_flip}
\end{figure}

\vspace{5pt}  

\begin{figure}[ht]
    \centering
    \includegraphics[width=\linewidth]{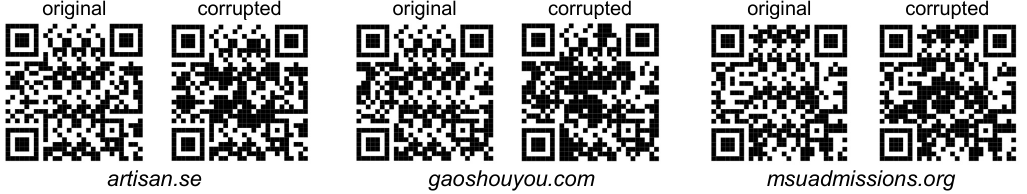}
    \vspace{-15pt}
    \caption{Examples where the Transformer successfully decodes QR codes after 10 burst errors.
``original'' shows the QR code before corruption, and ``corrupted'' shows the version after applying the errors.
While pyzbar fails to decode the corrupted QR code, the Transformer can decode it.}
    \label{appendix:sample_burst}
\end{figure}

Next, we examine the similarity between the plain text and the strings generated by the Transformer under severe corruption. 
Similarity was computed as the Levenshtein distance~\cite{Levenshtein} normalized by the maximum string length, as follows:
\begin{align}\label{appendix:similarity}
\mathrm{Similarity}(a, b) = 1-\frac{\mathrm{Levenshtein}(a,b)}{\mathrm{max}(|a|, |b|)}.
\end{align}
Here, $a$ and $b$ denote the two strings being compared.
\cref{appendix:fig:similarity} illustrates the distributions of similarity scores for the generated strings compared to the plain text when subjected to 20 flip errors and 10 burst errors, respectively.
In both (a) and (b), the Transformer produces a large proportion of strings with over 80\,\% similarity to the plain text, even under severe corruption. 
Moreover, applying data augmentation substantially increases these similarity scores. 
Therefore, the Transformer can generate highly similar outputs despite heavy damage, and its performance can be further enhanced through data augmentation. 
For reference, \cref{appendix:tab:string_similarity_examples} presents the correspondence between similarity scores and the actual generated strings.

\begin{figure}[ht]
    \centering
    \includegraphics[width=\linewidth]{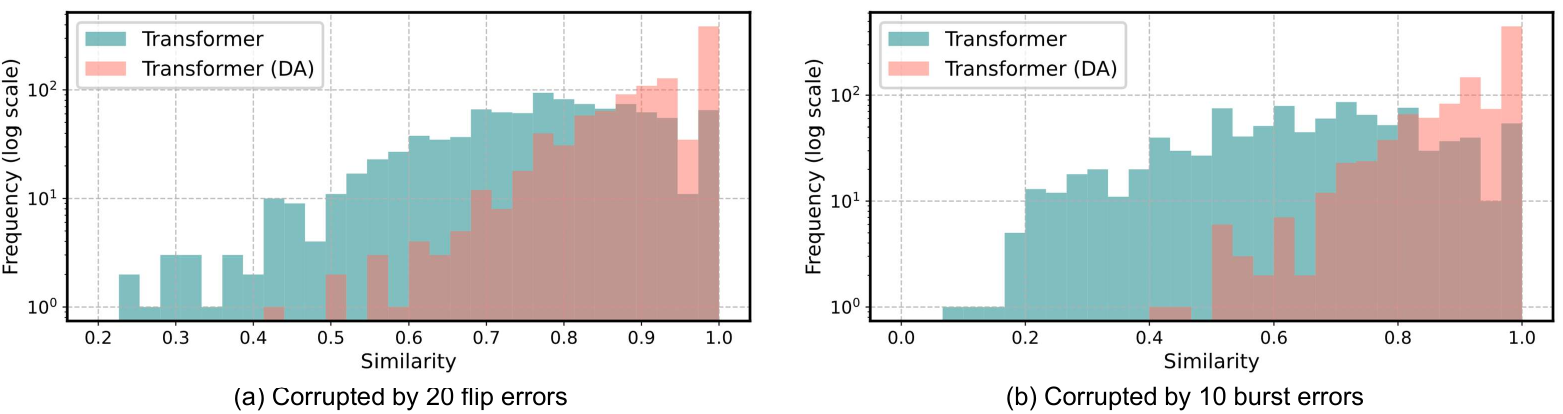}
    \vspace{-15pt}
    \caption{Distributions of similarity scores for the generated strings compared to the plain text when subjected to (a) 20 flip errors and (b) 10 burst errors. }
    \label{appendix:fig:similarity}
\end{figure}

\begin{table}[ht]
    \centering
    \caption{Examples of strings and their similarity scores.}
    \vskip 0.05in
    \label{appendix:tab:string_similarity_examples}
    \begin{tabular}{c*{6}{c}}
        \midrule
        Strings   & \texttt{Transformer} & \texttt{Transformey} & \texttt{Transforcel} & \texttt{Transfuimur} & \texttt{Tranyfjjber} & \\
        Similarity & 1.00 & 0.91 & 0.82  & 0.73 & 0.64     & \\
        \midrule
    \texttt{pcavsfzumec}  & \texttt{buagyfirchr}  & \texttt{Trpbgfbildv}  & \texttt{xwbiliorblq}  & \texttt{bpqzcvprwvh}  & \texttt{bxxggbggpjx} \\
    0.45  & 0.36    & 0.27  & 0.18  & 0.09  & 0.00 \\
        \bottomrule
    \end{tabular}
\end{table}

\section{Sensitivity of Decoding QR Code} \label{appendix:sensitivity}
In this section, we evaluate the sensitivity of QR code decoding.
The sensitivity of a function is defined by how much the output changes in response to variations in the input.
In this experiment, we examine how many bits in a QR code must be altered to produce a one-character change in the expected output (i.e., the plain text).
We prepare a set of domain names and generate 1,000 samples by randomly modifying a single character in each (e.g., changing ``example.com'' to ``exomple.com'').
We then convert both the original and modified strings into QR codes.
Finally, we count the number of differing bits between the original and modified QR codes and compute the average.

\cref{appendix:tab:sensitivity} shows the average number of bit changes resulting from a one-character change in the plain text for QR code Version 1 through Version 3. For Version 1 and Version 2, the number of bit changes required to alter the output increases with the error correction level. This suggests that the decoding of QR codes in Version 1 and Version 2 becomes less sensitive as the error correction level increases. On the other hand, Version 3 does not follow this trend. This is attributed to differences in the encoding schemes used in each version. In Version 3 and later versions, the data codewords are divided into two or more blocks, and each block is independently encoded with Reed–Solomon codes to generate corresponding error correction codewords. Because Version 3 and later versions perform encoding on split data blocks, the function sensitivity does not necessarily decrease with higher error correction levels. Given these factors, Version 1 and Version 2 are more suitable for examining how function sensitivity impacts the behavior of Transformers.

\begin{table}[htbp]
\setlength{\tabcolsep}{8pt}
\renewcommand{\arraystretch}{1.3}
\centering
\caption{Average number of bit changes caused by a single-character change in the plain text.}
\label{appendix:tab:sensitivity}
\vskip 0.05in
\begin{tabular}{ccccc}
\hline
Version & L & M & Q & H \\
\hline
1 & 30.3 & 42.0 & 54.1 & 69.2 \\
2 & 43.4 & 66.7 & 90.2 & 115.4 \\
3 & 62.1 & 106.5 & 74.3 & 92.1 \\
\hline
\end{tabular}
\end{table}

\section{Sensitivity Analysis Across QR Code's Four Error-Correction Levels}
In this section, we investigate the Transformer's behavior by training it to decode QR codes at four different error-correction levels, each corresponding to a distinct sensitivity. 

\cref{appendix:fig:graph_ce}(a) shows the success rate of the Transformers under flip errors. The results showed that there were almost no differences in success rate between error correction levels. For comparison, \cref{appendix:fig:graph_ce}(b) shows the success rate of Transformer pyzbar. It is clear that robustness against corruption varies significantly with the error-correction level. The performance gap between the Transformer and pyzbar likely arises because the Transformer focuses exclusively on the data codewords that carry the essential plain text information and thus ignores the error correction codewords.

\begin{figure}[ht]
    \centering
    \includegraphics[width=\linewidth]{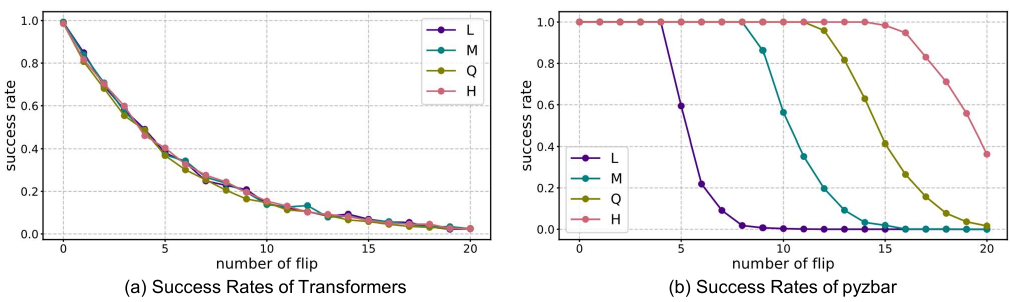}
    \vspace{-15pt}
    \caption{Success rates under flip errors for Transformer (a) and pyzbar (b), evaluated on v2-QR codes with all four error correction levels. The Transformer shows almost no difference in success rates across the error correction levels, whereas the pyzbar exhibits substantial variation depending on the level.}
    \label{appendix:fig:graph_ce}
\end{figure}

\section{Success Rate of Error Correction in the Encoding Region} \label{appendix:success_rate}

Here, we derive the success rate of error correction in QR codes with $n$-bit errors. We assume that a bit flip occurs uniformly at random in the encoding region. 
The encoding region consists of the data and error-correction codewords, the format information, and the remainder bits. Note that the remainder bits are not used in the decoding; thus, any errors in those bits do not affect the result.

Let $N$ denote the total number of bits in the encoding region, and let $N_{\mathrm{d}}$, $N_{\mathrm{f}}$ and $N_{\mathrm{r}}$ be the numbers of bits in the data and error-correction codewords and the format information, remainder bits, respectively. Accordingly, the numbers of erroneous bits are denoted by $p$, $q$, and $n-p-q$, respectively. Then, the successful rate of error correction is given by the following.
\ECCProbability*
In the following, we provide an overview of the derivation of $P_{\mathrm{d}}(p)$ and $P_{\mathrm{f}}(q)$.

\paragraph{$\bm{P_{\mathrm{d}}(p)}$ - Success Rate in Data and Error-Correction Codewords.}
For the data and error-correction codewords, let $M$ denote the total number of codewords, and let $M_{\mathrm{ecc}}$ represent the number of error-correction codewords, which depends on the error correction level. 
Here, $M = N_\mathrm{d} / 8$.
According to the properties of Reed–Solomon codes, the maximum number of correctable codewords, denoted by $t$, is determined by $t = \lfloor M_\mathrm{ecc} / 2 \rfloor$.
Under these conditions, the probability of successful error correction for the data and error-correction codewords can be determined as follows: 
\ProbabilityData*
\begin{proof}
Let $K$ denote the number of codewords that contain at least 1-bit error. Since the Reed–Solomon decoder can correct up to $t$ codewords with errors, the Success Rate is given by:
\begin{align}\label{data_prob1}
\begin{split}
    P_{\mathrm{d}}(p)
    &= P(K \le t) \\
    &= \sum_{k = \lceil \frac{p}{8} \rceil}^{t}P(K=k) \\
    &= \frac{1}{\binom{N_{\mathrm{d}}}{n}}\sum_{k = \lceil \frac{p}{8} \rceil}^{t}|\{K=k\}|.
\end{split}
\end{align}
Let $S_k(p)$ denote the number of ways in which exactly $k$ distinct codewords each contain at least 1-bit error, given that there are $p$ bit errors in total. Then, under the condition that all $k$ selected codewords include at least one erroneous bit, the value of $S_k(p)$ is given by
\begin{align}\label{data_prob2}
| \{K=k\} | = \binom{M}{k}S_k(p).
\end{align}
Let $U$ be the set of all ways to choose $n$ error-bit positions out of the $8k$ bits of the $k$ selected codewords (so $|U|=\binom{8k}{p}$).
 Define $A_i \subseteq U $ as the subset of error assignments in which the $i$-th codeword contains no bit errors. Then, the number of assignments in which every codeword contains at least 1-bit error is given by:
\begin{align}\label{data_prob3}
\begin{split}
    S_k(p) 
    &= \bigg| \bigcap^{k}_{i=1} A^{c}_{i} \bigg|  \\
    &= |U| - \bigg| \bigcup^{k}_{i=1} A_{i} \bigg|. 
\end{split}
\end{align}
To evaluate the number of assignments in which every codeword contains at least one erroneous bit, we apply the principle of inclusion-exclusion, as formalized in ~\cref{lemma:PIE}. Specifically, applying it to the sets $ A_1, A_2, \dots, A_k \subseteq U $, we obtain:
\begin{align}\label{data_prob4}
\begin{split}
    \bigg| \bigcup^{k}_{i=1} A_{i} \bigg| 
    &= \sum_{j=1}^{k}(-1)^{j-1} \sum_{L \subseteq [k],|L|=j} \bigg| \bigcap_{l \in L } A_{l} \bigg| \\
\end{split}
\end{align}
The term $ \bigg| \bigcap_{l \in L } A_{l} \bigg| $ denotes the number of cases in which $ j $ codewords, selected from the $ k $ total codewords, are all error-free; thus,
\begin{align}\label{data_prob5}
\begin{split}
    \sum_{L \subseteq [k],|L|=j} \bigg| \bigcap_{l \in L } A_{l} \bigg|
    &= \binom{k}{j} \binom{8(k-j)}{p}.
\end{split}
\end{align}
From \cref{data_prob3,data_prob4,data_prob5}, it follows that
\begin{align}\label{data_prob6}
\begin{split}
    S_k(p) 
    &= |U| - \bigg| \bigcup^{k}_{i=1} A_{i} \bigg| \\
    &= \binom{8k}{p} - \sum_{j=1}^{k}(-1)^{j-1} \binom{k}{j} \binom{8(k-j)}{p} \\
    &= \binom{8k}{p} + \sum_{j=1}^{k}(-1)^{j} \binom{k}{j} \binom{8(k-j)}{p} \\
    &= \sum_{j=0}^{k}(-1)^{j} \binom{k}{j} \binom{8(k-j)}{p}. \\
\end{split}
\end{align}
Therefore, from \cref{data_prob1,data_prob2,data_prob6}, we obtain
\begin{align}\label{data_prob7}
\begin{split}
    P_{\mathrm{d}}(p)
    &= \frac{1}{\binom{N_{\mathrm{d}}}{p}}\sum_{k = \lceil \frac{p}{8} \rceil}^{t}|{K=k}| \\
    &= \frac{1}{\binom{N_{\mathrm{d}}}{p}}\sum_{k = \lceil \frac{p}{8} \rceil}^{t} \binom{M}{k} \sum_{j=0}^{k}(-1)^{j} \binom{k}{j} \binom{8(k-j)}{p}.
\end{split}
\end{align}
\end{proof}

\paragraph{$\bm{P_{\mathrm{f}}(q)}$ - Success Rate in Format Information.}
The format information is encoded into 15 bits using a $(15, 5)$ BCH code and placed in two separate locations in the QR code.
If error correction works in either one of the two, the format information can be correctly read.
The $(15, 5)$ BCH code can correct up to 3-bit errors.
Under these conditions, the probability that the format information is correctly recovered is given below:
\ProbabilityFormat*
\begin{proof}
When exactly $q$ bit‐errors occur in the format information, there are $q+1$ ways to distribute those errors across the two $(15,5)$ BCH codeword blocks. Decoding of the format information succeeds if at least one of the two blocks contains at most 3-bit errors. 
Labeling by $i$ and $j$ the numbers of errors in the first and second blocks respectively, the number of allocations that lead to successful decoding can be written as $|\{(i,j)\in\mathbb{N}_0^2 \mid i+j=q,\ \min(i,j)\le3\}|$.
Hence,
\begin{align}
P_{\mathrm{f}}(q)
&= \frac{\bigl|\{(i,j)\in\mathbb{N}_0^2 \mid i+j=q,\ \min(i,j)\le3\}\bigr|}{q+1}.
\end{align}
\end{proof}

\section{Principle of Inclusion–Exclusion} \label{appendix:PIE_proof}
\begin{restatable}[name=Principle of Inclusion–Exclusion]{lemma}{PIE}\label{lemma:PIE}
Let $ A_1, A_2, \dots, A_k $ be subsets of a finite set $ U $. Then, the cardinality of their union is given by:
\begin{align}
\left| \bigcup_{i=1}^{k} A_i \right| 
= \sum_{j=1}^{k} (-1)^{j-1} \sum_{\substack{L \subseteq [k] , |L| = j}} \left| \bigcap_{l \in L} A_l \right|,
\end{align}
where $ [k] = \{1, 2, \dots, k\} $.
\end{restatable}
\begin{proof}
We prove the identity by counting how many times each element $ x \in U $ is counted on the right-hand side.
Fix an element $ x \in U $, and suppose that $ x $ belongs to exactly $ m $ of the sets $ A_1, A_2, \dots, A_k $. Without loss of generality, assume $ x \in A_1, A_2, \dots, A_m $, and $ x \notin A_{m+1}, \dots, A_k $.
On the right-hand side, $x$ contributes to each term of the form $\left| \bigcap_{l \in L} A_l \right|$, where $L \subseteq [k]$ and $x \in A_l$ for all $l \in L$. Since $x$ belongs to exactly $m$ of the $A_i$, the total number of such contributing terms is:
\begin{align} 
\begin{split}
    (-1)^{0}\binom{m}{1} + (-1)^{1}\binom{m}{2} + \cdots + (-1)^{m-1}\binom{m}{m}
= \sum_{j=1}^{m}(-1)^{j-1}\binom{m}{j}.
\end{split}
\end{align}
This sum can be evaluated using the binomial theorem:
\begin{align} \label{binomial}
\begin{split}
    0
    &= (1-1)^m \\
    &= \sum_{j=0}^{m} (-1)^j \binom{m}{j} \\
    &= 1 - \sum_{j=1}^m(-1)^{j-1}\binom{m}{j}.
\end{split}
\end{align}
From~\cref{binomial}, it follows that
\begin{align}
\begin{split}
    \sum_{j=1}^m(-1)^{j-1}\binom{m}{j}
    = 1.
\end{split}
\end{align}

Therefore, every $ x \in \bigcup_{i=1}^{k} A_i $ contributes exactly once to the right-hand side. Elements not in the union contribute zero. Hence, the right-hand side counts exactly the number of elements in the union.
\end{proof}

\end{document}